\pdfoutput=1

\documentclass[11pt]{article}

\usepackage{emnlp2021}

\usepackage{times}
\usepackage{latexsym}

\usepackage[T1]{fontenc}

\usepackage[utf8]{inputenc}

\usepackage{microtype}

\usepackage{amsmath}
\usepackage{amssymb}
\usepackage{graphicx}
\usepackage{makecell}
\usepackage{microtype}

\usepackage{tikz}
\usepackage{tikz-dependency}
\usetikzlibrary{shapes.geometric,calc}
\usetikzlibrary{colorbrewer}

\usepackage{caption}
\usepackage{subcaption}
\usepackage{multirow}
\usepackage{tabularx}
\usepackage{booktabs}
\usepackage[normalem]{ulem}
\useunder{\uline}{\ul}{}
\usepackage{float}
\usepackage{mathtools}
\usepackage{pgfplots}
\pgfplotsset{compat=1.14, every non boxed x axis/.append style={x axis line style=-},
     every non boxed y axis/.append style={y axis line style=-}}

\usepackage{ifthen}

\title{Pruning the Index Contents for Memory Efficient Open-Domain QA}

\usepackage{amsthm}
\newtheorem{theorem}{Theorem}[section]
\newtheorem{lemma}[theorem]{Lemma}

\usepackage{algorithm}
\usepackage[noend]{algpseudocode}

\usepackage{times,latexsym}
\usepackage{url}
\usepackage[T1]{fontenc}

\usepackage{amsmath}

\usepackage{lipsum}     
\usepackage{enumitem}

\title{R2-D2: A Modular Baseline for Open-Domain Question Answering}

\author{Martin Fajcik, Martin Docekal, Karel Ondrej, Pavel Smrz\\
  Brno University of Technology\\
  {\tt \{ifajcik,idocekal,ondrej,smrz\}@fit.vutbr.cz} }

\begin{document}
\maketitle

\begin{abstract}
    This work presents a novel four-stage open-domain QA pipeline R2-D2 (\textsc{Rank twice}, \textsc{reaD twice}). The pipeline is composed of a retriever, passage reranker, extractive reader, generative reader and a mechanism that aggregates the final prediction from all system's components. We demonstrate its strength across three open-domain QA datasets: NaturalQuestions, TriviaQA and EfficientQA, surpassing state-of-the-art on the first two. Our analysis demonstrates that: (i) combining extractive and generative reader yields absolute improvements up to 5 exact match and it is at least twice as effective as the posterior averaging ensemble of the same models with different parameters, (ii) the extractive reader with fewer parameters can match the performance of the generative reader on extractive QA datasets\footnote{Our demo is available at \url{http://r2d2.fit.vutbr.cz/}. Code and preprocessed data are available at \url{https://github.com/KNOT-FIT-BUT/R2-D2}.}.
\end{abstract}


    
\section{Introduction}

Last year showed rapid progress in neural factoid open-domain question answering based on \emph{retriever-reader} architecture (Open-QA).
Such Open-QA systems \cite{chen2017reading} seek evidence for answering the questions inside the knowledge source using the \emph{retriever} and then extract the answer from the retrieved knowledge using the \emph{reader}. The knowledge source is often a large corpus of short snippets of natural language, so-called passages (e.g., taken from an encyclopedia).

The progress can be attributed to advances in neural retrieval methods \cite[\textit{inter alia}]{karpukhin2020dense, izacard2020distilling, khattab2020relevance, luan2020sparse, xiong2020approximate} that benefit from smarter negative sampling strategies or a better trade-off between complex question-passage interaction and its efficiency. It also can be attributed to reading methods that enable processing large quantities of retrieved passages \citet{izacard2020leveraging}. They compensate for a certain amount of the retrieval error and enable early aggregation of answer's evidence between passages.

This work demonstrates the relative improvement of 23-32\% compared to last year's state-of-the-art DPR system  \cite{karpukhin2020dense}, while using the same knowledge source and the retriever. We propose a state-of-the-art Open-QA baseline composed of retriever, passage reranker, extractive reader, generative reader, and a novel component fusion approach. We follow the practice from information retrieval and show that our moderately sized reranker allows to reduce the passage count needed at the input of large reader models about four times. Our readers then take the best from both worlds. The extractive reader proposes a list of salient answer spans. The generative reader reranks these spans, seeing all the passages at once, or generates its own answer.
The proposed pipeline is heterogeneous and modular, making it an ideal benchmark. 

To sum up, our contributions are three-fold:
\begin{enumerate}
    \item We present a simple novel approach to aggregate scores from all system components and show that combining extractive and generative approaches is superior to a posterior averaging ensemble of homogeneous models.
    \item We show that the extractive reader can sometimes match the performance of the generative approaches without taking the advantage of the fusion between retrieved passages. This indicates that the evidence aggregation from multiple passages in the generative approaches is either not learned or not necessary to perform well on these datasets.
    \item We push the state-of-the-art for two large and popular datasets, demonstrating what is achievable with the proposed approach, having the same knowledge source and the retriever as in the previous works \cite{karpukhin2020dense,izacard2020leveraging}.
\end{enumerate}

\section{Open-QA Pipeline}
We propose the R2-D2 (\textsc{Rank twice}, \textsc{reaD twice}), 4-stage pipelined system that can choose whether to generate or to extract an answer.  The parameters of each component in pipeline are estimated separately. It is composed of DPR passage retriever \cite{karpukhin2020dense}, passage reranker (see subsection \ref{ss:reranker}), and two readers. Figure \ref{fig:r2d2_pipeline} shows the diagram of our system. The first reader performs an extractive span-selection similar to \citet{fajcik2020rethinking}. The second reader is based on Fusion-In-Decoder (FiD) \cite{izacard2020leveraging}.

Formally, given a question $q \in \mathcal{Q}$ from the set of all possible questions $\mathcal{Q}$  and the corpus $\mathcal{C}=\{p_1, p_2, ... , p_n\}$ composed of passages $p_i$, the retriever learns a ranking function $\operatorname{rank}:\mathcal{Q} \times \mathcal{C} \rightarrow \mathbb{R} $ that assigns a score to each passage. We assume each passage contains its passage title (e.g., title from the Wikipedia article).

Taking a top-$K$ scoring passages $\mathcal{C}_{r}\subset\mathcal{C}$, reranker again rescores $\mathcal{C}_{r}$  scoring passages by learning a reranking function $\operatorname{rerank}:\mathcal{Q} \times \mathcal{C}_{r} \rightarrow \mathbb{R}$. Note that while $\operatorname{rank}$ and $\operatorname{rerank}$ have similar signatures, the computational cost of $\operatorname{rerank}$ over the same amount of passages is drastically higher, as it computes fine-grained interaction between tokens of question and passage.

Next, the rescored passages are passed to two readers: the extractive reader reads top-$V$ passages $\mathcal{C}_{rr}\subset\mathcal{C}_{r}$ independently of each other and assigns the probability $\boldsymbol{P}_{e}(a_e|q, \mathcal{C}_{rr})$ to each span $a_e$ in the passages (see subsection \ref{ss:ext_reader}). 
The FiD generative reader reads top-$V_2$ passages $\mathcal{C}_{rr}'\subset\mathcal{C}_{r}$  jointly and generates an answer from probability space $\boldsymbol{P}_g(a_g|q,\mathcal{C}_{rr}')$ via greedy search.

Finally, R2-D2 aggregates the outputs from all components using two fusions (described in subsection \ref{ss:fusions}).

\begin{figure}[t!]
    \centering
    \def\col{2.00cm}
\def\row{3.30cm}
\def\blockwidth{3.1cm}
\def\blockwidthlarge{7.1cm}
\def\blockheight{0.90cm}
\def\arrowrate{0.3}

\def\outputcolor{gray!20}
\def\examplecolor{white}

\newcommand\Textbox[2]{%
    \parbox[c][\dimexpr#1-7.7pt][c]{2.3cm}{\centering{\bf #2}}}

\begin{tikzpicture}[->,line width=2pt,>=latex,node distance=\lvspace,
                    font=\footnotesize,
                    thick,
                    module/.style={
                        line width=1pt,
                        rectangle split, 
                        rectangle split parts=3,
                        rectangle split part fill={blue!30, gray!20, gray!10},
                        draw,
                        text width=\blockwidth, 
                        align=center, 
                        rounded corners=0.1cm,
                    }
                ]
    \node[align=center,text width=0.5\textwidth] at (1*\col,3.7*\row) (question) {\large In which Czech city is the brewery of its largest beer exporter?};

    \node[module,rectangle split part fill={blue!30, \outputcolor, \examplecolor}, text width=\blockwidthlarge]   at (1*\col,3*\row) (R)  {\Textbox{\blockheight}{Retriever}\nodepart{two}{top-K passages}\nodepart[align=left]{three}{1. ... town of České Budějovice, known as Budweis... \\ 2. Czech Beer Festival is the biggest ... \\ 3. Plzeň, also called Pilsen is a city...}};
    \node[cylinder, draw, ,thick,aspect=0.2,
        minimum height=0.8cm,minimum width=1.0cm,
        shape border rotate=90, align=center] at (-0.5*\col,3.2735*\row) (datastorage) {index};

    \node[module,rectangle split part fill={cyan!40, \outputcolor, \examplecolor}, text width=\blockwidthlarge]   at (1*\col,2*\row) (RR) {\Textbox{\blockheight}{Passage reranker}           \nodepart{two}{top-K reranked passages}\nodepart[align=left]{three}{1. Plzeň, also called Pilsen is a city... \\ 2. ... town of České Budějovice, known as Budweis... \\ 3. Czech Beer Festival is the biggest ...}};
    \node[module,rectangle split part fill={green!30, \outputcolor, \examplecolor}]  at (0*\col,1*\row) (ER) {\Textbox{\blockheight}{Extractive reader}          \nodepart{two}{top-M answer spans}\nodepart[align=left]{three}{1. České Budějovice  \\ 2. Festival \\ 3. Plzeň}};
    \node[module,rectangle split part fill={yellow!40, \outputcolor, \examplecolor}] at (2*\col,1*\row+2.1ex) (AR) {\Textbox{\blockheight}{Abstractive reader}         \nodepart{two}{top generated answer}\nodepart[align=left]{three}{1. Brno}};
    \node[module,rectangle split part fill={yellow!40, \outputcolor, \examplecolor}]    at (0*\col,0*\row) (ARR)  {\Textbox{\blockheight}{Abstractive reader}\nodepart{two}{top-M reranked spans}\nodepart[align=left]{three}{1. Plzeň \\ 2. Festival \\ 3. České Budějovice}};
    \node[module,rectangle split part fill={red!30, \outputcolor, \examplecolor},]    at (0*\col,-1*\row) (AGGR)  {\Textbox{\blockheight}{Score aggregation}\nodepart{two}{top-M aggr. spans}\nodepart[align=left]{three}{1. Plzeň \\ 2. České Budějovice \\ 3. Festival}};
    \node[module,rectangle split part fill={red!40, \outputcolor, \examplecolor}]    at (1*\col,-2*\row+2.1ex) (BD)  {\Textbox{\blockheight}{Binary decision}\nodepart{two}{top answer}\nodepart[align=left]{three}{1. Plzeň}};
    
    \draw[->,rounded corners] (question.south) to (R.north-|question.south);
    \draw[->,rounded corners] (R) to (RR);
    \draw[->,rounded corners] (RR) to ($(RR.south)!\arrowrate!(ER.north-|RR.south)$) to ($(ER.north|-RR.south)!\arrowrate!(ER.north)$) to (ER.north);
    \draw[->,rounded corners] (RR) to ($(RR.south)!\arrowrate!(AR.north-|RR.south)$) to ($(AR.north|-RR.south)!\arrowrate!(AR.north)$) to (AR.north);
    \draw[->,rounded corners] (ER) to (ARR);
    \draw[->,rounded corners] (AR) to ($(AGGR.south-|AR.south)!\arrowrate!(BD.north-|AR.south)$) to ($(BD.north|-AGGR.south)!\arrowrate!(BD.north)+(0.5cm,0)$) to ($(BD.north)+(0.5cm,0)$);
    \draw[->,rounded corners] (RR) to ($(RR.south|-ARR.east)+(0,0.9cm)$) to ($(ARR.east)+(0,0.9cm)$);
    \draw[->,rounded corners] (RR) to ($(RR.south|-AGGR.east)+(0,0.9cm)$) to ($(AGGR.east)+(0,0.9cm)$);
    \draw[->,rounded corners] (ARR) to (AGGR);

    \draw[->,rounded corners] (AGGR) to ($(AGGR.south)!\arrowrate!(BD.north-|AGGR.south)$) to ($(BD.north|-AGGR.south)!\arrowrate!(BD.north)-(0.5cm,0)$) to ($(BD.north)-(0.5cm,0)$);
    
    \draw[->,rounded corners] (R) to ($(R.south)!\arrowrate!(RR.north-|R.south)$) to ($(R.south)!\arrowrate!(RR.north-|R.south)-(4.22cm, 0)$) to ($(AGGR)-(2.22cm,-0.7cm)$) to ($(AGGR.west)+(0,0.7cm)$);
    
    \draw[->,rounded corners] (ER) to ($(ER.south)!\arrowrate!(ARR.north-|ER.south)$) to ($(ER.south)!\arrowrate!(ARR.north-|ER.south)-(2.00cm, 0)$) to ($(AGGR)-(2.00cm,-1.1cm)$) to ($(AGGR.west)+(0, 1.1cm)$);
\end{tikzpicture}
    \caption{R2-D2 pipeline.}
    \label{fig:r2d2_pipeline}
\end{figure}
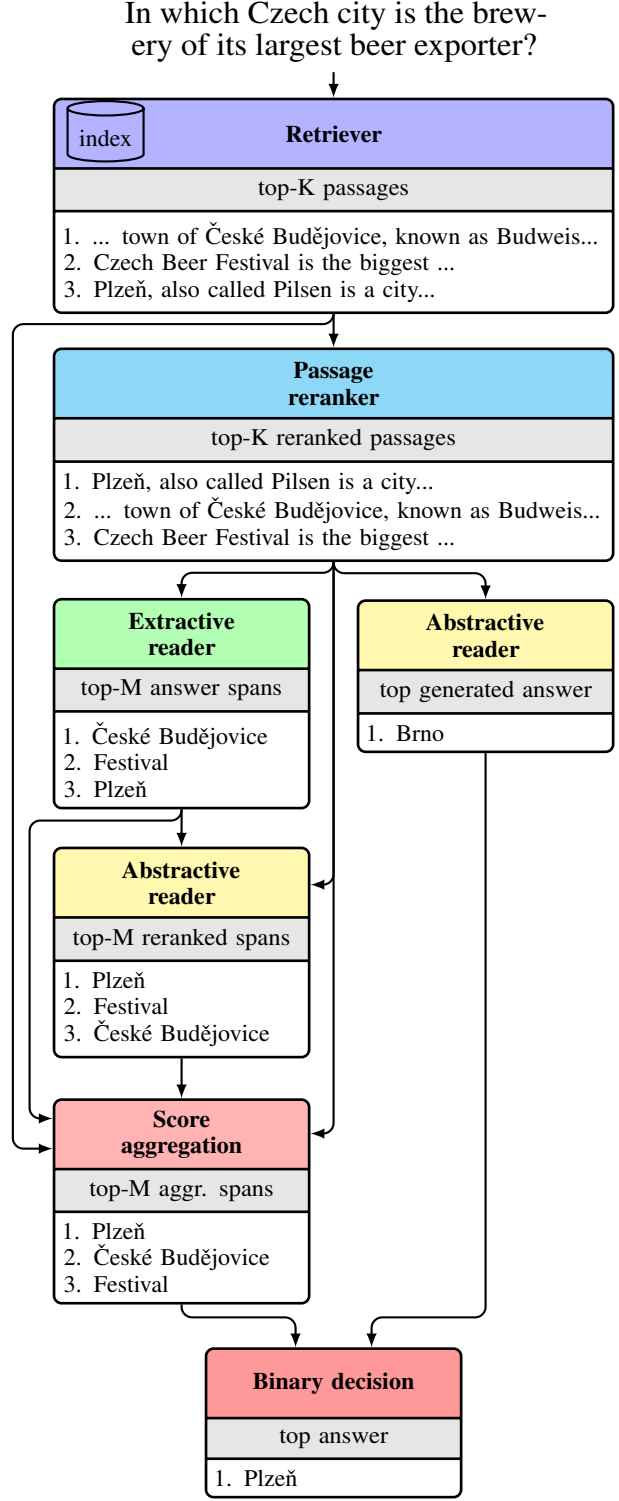

\subsection{Passage Reranker}\label{ss:reranker}
The proposed passage reranker is based on transformer cross-encoder similar to \citet{nogueira2019passage, luan2020sparse}.
The input is the concatenation of question $q\in\mathcal{Q}$ and passage $p\in\mathcal{C}_r$ with a special \texttt{SEP} token between them. The passage consists of a title and context that are prepended with special start tokens and concatenated together.
We denote the contextual representation of input token $w$ obtained by the cross-encoder as $\operatorname{En}(p, q)[w]\in\mathbb{R}^d$.

    Now we can define the reranking function for passage rescoring as
\begin{equation}
    \operatorname{rerank}(q,p) = \operatorname{En}(p,q)[\texttt{CLS}]^\top w
\end{equation}
where $w\in\mathbb{R}^{d}$ is a trainable vector and \texttt{CLS} is the special token added at the start of an input sequence.
Finally, we define the following formula\footnote{Formal definition of softmax over a set is described in the Apendix \ref{app:softmax_not}.}
\begin{equation}
    \boldsymbol{P}_{rr}\left(p | q, \mathcal{C}_r\right) =\operatorname*{softmax}\limits_{p\in \mathcal{C}_r}\left(\operatorname{rerank}\left(q, p\right)\right)_{p}
\end{equation}
to assign a probability to the case that passage $p$ contains answer to the question $q$.

\begin{description}[style=unboxed,leftmargin=0em,listparindent=\parindent]
    \setlength\parskip{0em}
\item[Training.] The model input for each question is exactly one positive sample supplemented with hard negatives from the retriever. The ground truth passage, annotated the same way as in \citet{karpukhin2020dense}, is primarily used as a positive sample. If the ground truth is unknown, the positive sample is the best retriever passage containing the answer.
The hard negatives are uniformly sampled from retriever's top-$K$ results that do not contain the answer. 
The used loss function is the cross-entropy.
\end{description}

\subsection{Extractive Reader}
\label{ss:ext_reader}

Extractive reader estimates the probability $\boldsymbol{P}_{e}(a_e|q, \mathcal{C}_{rr})$. 
It is the probability of a span $a_e$ from top-$V$ passage $p \in \mathcal{C}_{rr}$ being an answer to a question $q$. 
We decompose the $\boldsymbol{P}_{e}(a_e|q, \mathcal{C}_{rr})$ into four probabilities of:

\begin{itemize}
    \setlength\parskip{0em}
    \item token $s$ being starting token of an answer span,
    \item token $e$ being ending token of an answer span,
    \item tokens $s$ and $e$ being boundary tokens of an answer span \cite{fajcik2020rethinking},
    \item passage $p$ containing an answer for the question $q$ (inner reranker) as in \citet{karpukhin2020dense}.
\end{itemize}


To obtain the final probability used in test-time, we compute their product\footnote{We tried decoding from the subsets of these probabilities in Appendix \ref{app:decoding_ext_probs} not observing significant difference.}. These probabilities are defined as:
\begin{equation}
    \boldsymbol{P}_{*}(*|q, \mathcal{C}_{rr}) = \operatorname{softmax}(s_*)_i \: ,
\end{equation}
where $*$ may stand for a \emph{start}, \emph{end}, \emph{joint}, and a \emph{passage}. The $i$ is an index of a given element, and the $s_*$ is a vector of scores for each element among all passages in $\mathcal{C}_{rr}$. So the \emph{softmax} normalization sum goes through all the passages. On the other hand, the $s_*$ scores are estimated by the model with just a single passage at its input \cite{clark-gardner-2018-simple}. The scores are as follows:
\setlength{\jot}{1ex}
\begin{gather}%
s^{i}_{start} = \operatorname{En}(p,q)[s]^\top w_{start} \\
s^{i}_{end} = \operatorname{En}(p,q)[e]^\top w_{end} \\
s^{i}_{joint} = (W_j \operatorname{En}(p,q)[s] + b_j)^\top \operatorname{En}(p,q)[e] \\
s^{i}_{passage} = \operatorname{En}(p,q)[\texttt{CLS}]^\top w_{p} \:.
\end{gather}%
Where $w_*,b_j \in \mathbb{R}^h$, $\operatorname{En}(p, q)[\cdot] \in \mathbb{R}^h$, and $W_j \in \mathbb{R}^{h \times h}$ are all trainable parameters.

We omit the spans of a title and question for answer span selection. Therefore the final answer can be selected only from the context.

The following training objective with independently marginalized components is used:
\begin{equation} \label{eq:extReaderIndLoss}
\begin{split}
	-\log \sum_{s \in starts(C_{rr})} \boldsymbol{P}_{start}(s|q, \mathcal{C}_{rr}) \\
	-\log \sum_{e \in ends(C_{rr})} \boldsymbol{P}_{end}(e|q, \mathcal{C}_{rr}) \\
	-\log \sum_{j \in boundaries(C_{rr})} \boldsymbol{P}_{joint}(j|q, \mathcal{C}_{rr}) \\
	-\log \sum_{p \in C_{rr}} \boldsymbol{P}_{passage}(p|q, \mathcal{C}_{rr}) \: .
\end{split}
\end{equation}

The sums are going through target annotations (starts, ends, etc.) obtained by the distant supervision approach.
\subsection{Component Fusion}
\label{ss:fusions}
To produce the final answer, R2-D2 aggregates the log-probabilities of all system components via linear combinations tuned on validation data.  

Firstly, the log-probabilities of all system components for top-$M$ answer spans proposed by the extractive reader are aggregated. Formally, assume the $\mathcal{A}_q$ is the set of top-$M$ answer spans from $\boldsymbol{P}_{e}(a|q,\mathcal{C}_{rr})$ for question $q$.
The generative model performs  the \textbf{answer reranking} evaluating the log-probability of the answer spans 
\begin{equation}
    \label{eq:genrerank}
    \{\log\boldsymbol{P}_g(a|q,\mathcal{C}_{rr}'): a\in \mathcal{A}_q\}.
\end{equation}

Next a logistic regression loss \eqref{eq:aggloss} is minimized to perform \textbf{score aggregation}. 
It combines the scores across the R2-D2 components to maximize the correct answer span probability over dataset~$\mathcal{D}$. This dataset is composed of the top-$M$ outputs of the extractive reader with the correct answer.
\begin{gather}%
x(a) =[\boldsymbol{P}_{e}(a) \; \boldsymbol{P}_g(a) \; \boldsymbol{P}_r(p_a) \; \boldsymbol{P}_{rr}(p_a)] \\
\label{eq:aggloss}
   -\sum_{\mathclap{(\mathcal{A}_q,gt) \in \mathcal{D}}} \log\operatorname*{softmax}\limits_{a \in \mathcal{A}_q} \big({ w^\top \log x(a) + b}\big)_{gt}
\end{gather}%
Here $p_a$ denotes the passage containing the answer span $a$, $\mathcal{A}_q$ is a set of proposed answer spans, $gt$ is the correct answer span, distribution dependencies are dropped for clarity and only the logistic regression parameters $w, b$ are tuned in this step.

Finally, we theorized the correct answer span might not always be available in the passage set $\mathcal{C}_{rr}$, but the generative reader might be able to generate the answer from its parameters and the evidence given in passages. We introduce the binary classifier, which decides whether to select the best span answer from answer aggregation step or a free-form answer generated via FiD. Given that $s_{agg}(q)=\max_{a\in\mathcal{A}_q} w^\top x(a)+b$ is the best span score and $s^*_g(q)=\log\boldsymbol{P}_g(a_q^*|q,\mathcal{C}_{rr}')$ is the log-probability of the answer $a_q^*$  obtained via greedy decoding for question $q$, a classifier is trained via binary cross-entropy $BCE(l,t)$ with log-odds ratio $l$ and target $t$ to do the \textbf{binary decision}
\begin{equation}
\label{eq:bdformula}
       \sum_{(e,t) \in \mathcal{D}} BCE(w^\top [s_{agg}(e);s^*_g(e)]+b, t ).
\end{equation}
Here, the training dataset $\mathcal{D}$ contains only cases where either the extractive or the abstractive prediction is correct (but not both).
\section{Experimental Setup}
Our models are implemented in PyTorch \cite{paszke2019pytorch} using Transformers \cite{wolf-etal-2020-transformers}. We use 12GB GPU to train the passage reranker, 48GB GPU for the generative reader, and 16x 32GB GPUs to train the extractive reader with $V=128$ passages at its input. The inference runs on 12GB GPU. In all experiments, we used Adam optimizer with a decoupled weight decay \cite{loshchilov2017decoupled}. Our models are evaluated by two metrics:
\begin{description}[style=unboxed,leftmargin=0em,listparindent=\parindent]    \setlength\parskip{0em}
    \item [Exact match (EM)] measures the proportion of examples, for which the system prediction matched at least one annotated ground-truth answer. We use the script from \citet{lee-etal-2019-latent}\footnote{\url{https://cutt.ly/rkZNIer}}.
    \item [Accuracy@K] measures the proportion of examples, for which the ground-truth answer string is present in top-K retrieved passages. We match the string exactly as \citet{karpukhin2020dense}\footnote{\url{https://cutt.ly/0luNhx4}}.
\end{description}

\subsection{Datasets and Data Pre-processing}

We evaluate our models on three datasets. Their statistics are available in Table \ref{fig:datatsets}. To train the reranker we filter out examples, which do not contain golden passage or exact match in top-$K$ retrieved passages. 
To train the extractive reader, only examples with exact match in a golden passage or top-1 retrieved passage are kept. Both filtering strategies are closely described in Appendix~\ref{app:data_preprocessing}.

\begin{description}[style=unboxed,leftmargin=0em,listparindent=\parindent]
\setlength\parskip{0em}
\item[NQ-Open] \cite{kwiatkowski2019natural, lee-etal-2019-latent} or NaturalQuestions-Open consists of real user queries obtained from Google search engine. The maximum length of each answer is at most 5 tokens. Each training and development sample contains 1 annotated answer, while test data contain 5-way answer annotation.

\item[TQ-Open] \cite{joshi-etal-2017-triviaqa} or TriviaQA-Open consists of question-answer pairs from 14 different trivia quiz websites. Each question contains human annotated answer and a set of answer aliases gathered from Wikipedia. We use the unfiltered version.

\item[EfficientQA] \cite{min2021neurips} is a dataset collected the same way as NQ-Open through 2019 and thus may contain more questions without evidence in our corpus than NQ-Open. We use the officially released dev set for testing\footnote{The test set was not released during our experiments.} models trained on NQ-Open training data.
\end{description}
Additionally, we also report results according to train-test set overlaps discovered by \citet{lewis-etal-2021-question} in Appendix \ref{app:train_test_overlap}.
\begin{table}
    \centering
    \scalebox{0.92}{\def\tablepad{1ex}
\begin{tabular}{l r r r}
\toprule
    \multicolumn{1}{c}{\bf Dataset} & \multicolumn{1}{c}{\bf Train} & \multicolumn{1}{c}{\bf Dev} &     \multicolumn{1}{c}{\bf Test}  \\
    \midrule
    NQ-Open    & 79,168 & 8,757 & 3,610 \\
    \hspace{\tablepad} 
     - filt. reranker    & 71,238 & - & - \\
    \hspace{\tablepad}  
     - filt. ext. reader   & 61,755 & - & - \\
    \hspace{\tablepad} 
     - w/ golden passage  & 58,876 & 6,515  & - \\
    TQ-Open    & 78,785 & 8,837 & 11,313 \\
    \hspace{\tablepad} 
     - filt. reranker    & 69,346 & - & - \\
    \hspace{\tablepad} 
     - filt. ext. reader   & 62,332 & - & - \\
    \hspace{\tablepad} 
     - w/ golden passage    & 60,413 & 6,760 & - \\
    EfficientQA & -      & -       & 1,800 \\
\bottomrule
\end{tabular}}
    \caption{Dataset statistics. The filt. lines report how many examples are kept for training the reranker (filt. reranker) and extractive reader (filt. ext. reader). The lines w/ golden passage denote how many examples from the set contain golden passage annotation. }
    \label{fig:datatsets}
\end{table}
\subsection{Models and Pipeline}
\label{ss:models_and_pipeline}
\begin{description}[style=unboxed,leftmargin=0em,listparindent=\parindent]
\setlength\parskip{0em}


\item[Retriever.]
We use BERT-based DPR from the official checkpoint\footnote{\url{https://github.com/facebookresearch/DPR}}. Each passage is represented via 768-dimensional embedding. We use a multiset checkpoint for TQ-Open, as the checkpoint for TQ directly isn't officially released.
We use the same knowledge corpus containing 21,015,320 passages based on 12-20-2018 Wikipedia snapshot as \citet{karpukhin2020dense}. In inference time, the retriever passes $K=200$ passages $\mathcal{C}_r$ to reranker.

\item[Passage reranker.]
We use the RoBERTa-base \cite{liu2019roberta} and truncate the inputs to a maximum length of 256. The linear scheduler with 0.1 warmup proportion is used, the number of epochs is 5 and the model is validated every 40,000 optimization steps. 
We use learning rate $1.6\cdot10^{-4}$ and batch size 8. In training, the model reranks 24 passages per question with negatives uniformly sampled from top-400 passages retrieved by DPR.
During the inference, top-$K$ ($K=200$) retriever passages are rescored and passed to readers.

\item[Extractive reader.] The extractive reader encoder is based on pre-trained ELECTRA-large. 
Its inputs are truncated if they are longer than the allowed maximum size (512 tokens). 
During the training phase, all spans from all $p \in \mathcal{C}_{r}$\footnote{Note that we train on data from retriever, not reranker.} that match\footnote{Matching strategies are described in Appendix \ref{app:data_preprocessing}.} with at least one of the known answers are selected as target annotations. 
Therefore the annotations might appear in the wrong context. 
The extractive reader reads the top-$V = 128$ passages during the training phase and when it is used without the reranker. 
To demonstrate the effect of reranker, the reader reads only the top-$V = 24$ passages if the reranker is used.
We use a linear scheduler with a warmup for the first 20,000 steps for all models. 
The maximum number of training steps is 200,000. 
The model is validated every 20,000 steps, and the best checkpoint among validations is selected. 
The learning rate is $2 \cdot 10^{-5}$ and the optimization step was done after each training example.

\item[Generative reader.]
We utilize T5-large \cite{raffel2020exploring} and use a concatenation of question, passages and their respective titles at the Fusion-in-Decoder's input the same way as \citet{izacard2020distilling}.  
We truncate each passage to the length of 250 tokens for NQ. 
For TQ, as questions are significantly longer, we truncate whole inputs to the same size.
Following FiD for TQ, we use only human-generated answer. 
In training, the golden passage always comes first, if available, and we take the rest of passages as ranked by retriever up to $V_2$ passages.
\citet{izacard2020leveraging} trained FiD with $V_2 = 100$ passages at its input. However, such approach requires tremendous amount of GPU memory, and thus requires employing speed-memory trade-offs such as gradient checkpointing \cite{Chen2016TrainingDN}. Unlike the original approach, we use only $V_2 = 25$ passages in our FiD. We note that in practice combining reranker with shorter-context FiD yields results similar to original implementation with much lower memory consumption and better throughput in the R2-D2 setting\footnote{Due to the numerous decoder computations in answer re-ranking.}. We analyze the speed of our implementation in Appendix \ref{app:inference time}. 
Other hyperparameters are similar to the original work---batch size 64, learning rate $5 \cdot 10^{-5}$ but no learning rate schedule. 
In test time, we decode an answer via greedy decoding.
\end{description}

\section{Results and Analysis}

\begin{table}[t]
    \scalebox{0.63}{\begin{tabular}{llccc}
\cmidrule[\heavyrulewidth]{2-5}
\multicolumn{1}{c}{} & \multicolumn{1}{c}{\textbf{Method}} & \multicolumn{1}{c}{\textbf{NQ}}& \multicolumn{1}{c}{\textbf{TQ}} & \multicolumn{1}{c}{\textbf{\#$\boldsymbol{\theta}$}} \\ \cmidrule{2-5} 
\multirow{13}{*}{\rotatebox[origin=c]{90}{Extractive}}   & BM25+BERT \cite{mao2020generation}&    37.7    &    60.1                     &  110M                 \\
                     & Hard EM \cite{min2019discrete}      &    28.1   &    50.9                   &  110M                 \\
                     & Path Retriever \cite{asai2019learning}&  32.6   &    -                      &  447M                 \\ 
                     & Graph Retriever \cite{min2019knowledge}& 34.5   &    56.0                   &  110M                 \\
                     & ORQA \cite{lee-etal-2019-latent}           &    33.3  &    45.0             &  220M                 \\
                     & REALM \cite{guu2020realm}           &    40.4     &    -                    &  660M                 \\
                     & ProQA \cite{xiong2020progressively} &    34.3     &    -                    &  220M                 \\
                     & DPR \cite{karpukhin2020dense}       &    41.5     &    56.8                 &  220M                 \\
                     & RDR \cite{yang2020retriever}        &    42.1    &   57.0                   &  110M                \\
                     & GAR+DPR \cite{mao2020generation}    &    43.8    &    -                     &  626M                \\ 
                     & ColBERT \cite{khattab2020relevance} & 48.2  &    63.2$^{-}$                 &  440M                \\ 
                     & RIDER (GAR+DPR) \cite{mao2021reader} &    48.3  &    -                      &  626M                \\ 
                     & UnitedQA-E  \cite{cheng2021unitedqa}                         &    51.8  &    68.9                    &  440M                \\\cmidrule{2-5} 

\multirow{8}{*}{\rotatebox[origin=c]{90}{Generative}}      & BM25+SSG \cite{mao2020generation}&35.3 &    58.6                    &  406M \\
                     & T51.1+SSM \cite{roberts2020much}    &    35.2   &    61.6                    &  11B                  \\
                     & RAG \cite{lewis2020retrieval}       &    44.5   &    56.8                    &  516M                 \\ 
                     & DPR+SSG \cite{min2020ambigqa}       &    42.2   &    -                       &  516M                 \\ 
                     & FiD-base \cite{izacard2020leveraging}&   48.2   &    65.0                    &  333M                 \\ 
                     & FiD-large \cite{izacard2020leveraging}&  51.4   &    67.6                    &  848M                 \\ 
                     & FiD-large++ \cite{izacard2020memory}&    54.7   &    \textbf{73.3}           &  848M                 \\
                     & UnitedQA-G \cite{cheng2021unitedqa}                           &    52.3   &   68.6                     &  880M                \\\cmidrule{2-5}
                     & UnitedQA (Ens. E+G+G) \cite{cheng2021unitedqa}           &    54.7   &   70.5                     &  1.87B                \\\cmidrule{2-5} 

\multirow{4}{*}{\rotatebox[origin=c]{90}{Ours}}            & R1-D1 (Generative)              &    49.9        &    65.4                     &  848M                 \\
                     & R1-D1 (Extractive)                  &    50.8     &    65.0                    &  445M                 \\ 
                     & R2-D2 (21M)                         &    \textbf{55.0}   &    69.9             &  1.29B                     \\
                     & R2-D2 (21M) w/ HN-DPR               &    \textbf{55.9}   &    -             &  1.29B                     \\ \cmidrule[\heavyrulewidth]{2-5} 
\end{tabular}}
    \caption{Comparison with the state-of-the-art in EM. \#$\theta$ denotes the estimated amount of model parameters. 
    Symbol $^{-}$ reports the result only for smaller system with $220M$ parameters.  
    }
    \label{tab:systems}
\end{table}

\begin{table*}[ht]
    \centering
    \scalebox{1.00}{\begin{tabular}{cc|rrr|rrr|rrr}
\toprule
\multirow{2}{*}{\textbf{Readers}} & \multirow{2}{*}{\textbf{Fusion}} &  \multicolumn{3}{c|}{\textbf{NQ-Open}}           & \multicolumn{3}{c|}{\textbf{TQ-Open}}           & \multicolumn{3}{c}{\textbf{EfficientQA}}             \\
& & \multicolumn{1}{c}{ret.}     & \multicolumn{1}{c}{+rr}     & \multicolumn{1}{c|}{$\Delta$} & \multicolumn{1}{c}{ret.}     & \multicolumn{1}{c}{+rr}     & \multicolumn{1}{c|}{$\Delta$} & \multicolumn{1}{c}{ret.}     & \multicolumn{1}{c}{+rr}     & \multicolumn{1}{c}{$\Delta$} \\ \midrule
ext                               & -                                & 50.78 & 50.72          & -0.06                         & 65.01 & 65.46          & 0.45                          & 47.00 & 47.56          & 0.56                         \\
gen                               & -                                & 49.92 & 50.69          & 0.77                          & 65.38 & {\ul 69.14}    & 3.76                          & 44.83 & 47.33          & 2.50                         \\
ext+gen                           & naive                            & 51.88 & 52.44          & 0.56                          & 66.17 & 68.01          & 1.84                          & 47.06 & 49.11          & 2.05                         \\
ext+gen                           & aggr                             & 54.13 & 54.90          & 0.77                          & 67.42 & 68.66          & 1.24                          & 50.44 & 52.00          & 1.56                         \\
ext+gen                           & aggr+bd                          & 54.07 & \textbf{54.99} & 0.92                          & 67.37 & \textbf{69.94} & 2.57                          & 49.72 & \textbf{52.22} & 2.50                         \\ \bottomrule
\end{tabular}}%
    \caption{Ablation study. We report results for extractive (ext), generative (gen) and both readers (ext+gen) without (ret.) and with reranking (+rr). The $\Delta$ column shows the exact match difference caused by passage reranking.}
    \label{tab:ablation_study}
\end{table*}



The effectiveness of our approach is compared with the state-of-the-art in Table \ref{tab:systems}. 
Our system, composed of just the retriever and FiD reader R1-D1 (Generative), shows inferior performance compared to FiD-large. 
This is most likely caused by 4 times fewer passages at its input, as in \citet{izacard2020leveraging}. 
In contrast, our ELECTRA based extractive reader R1-D1 (Extractive) shows large gains compared to extractive state-of-the-art, while having the same retriever as DPR.
We hypothesize this may be caused by ELECTRA pre-training method, which shows strong performance through variety of tasks and we further show that it is also due to training and inference with large input size of 128 passages and better objective (discussed in Section \ref{sec:ext_reader_perf} and Appendix \ref{app:ext_r_ablations}).
Only system that matches the performance of our extractive reader is the concurrent work on UnitedQA-E \cite{cheng2021unitedqa}, which uses advanced regularization and HardEM techniques. We note that these are orthogonal to our approach and could potentially lead to further improvements.

Finally, we find that our R2-D2 system with 21M passages corpus is competitive even with FiD++, which uses DPR retriever improved via knowledge distillation, and 26M passage corpus, which also includes lists. 
Additionally, we evaluate our model with a better retrieval model (HN-DPR) based on the DPR checkpoint where hard negatives are mined using the retrieval model itself\footnote{\url{https://cutt.ly/Ux5Yt4h}}.
Note that we do not compare EfficientQA with state-of-the-art, as the previous works didn't reported results on dev set we use for testing.

\subsection{Reranker Performance}
Next, we compare the performance of our retriever, reranker and reader with Accuracy@K in Figure~\ref{fig:acc_k_nq}. 
The passage reranker improves the accuracy consistently and we observe the same trend on other datasets (Appendix \ref{app:accuracy_at_k}).
We also include analysis, where we rerank each passage $p_i$ according its $s^i_{passage}$ score from extractive reader.
We observe results similar or even better to reranker for $K<10$, indicating the extractive reader reranks well on its own.
However, in subsequent experiments we do not replace the reranker with reader because: (i) passage reranker has fewer parameters, (ii) extractive reading can run in parallel with reranking and generative reading as extractive reader is not benefiting from reranking, and (iii) passage reranking scores often improve results during score aggregation (see Section \ref{sec:component_fusion}).

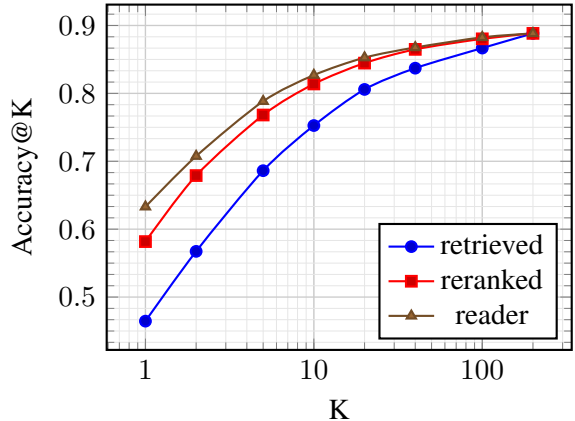
\begin{figure}[t]
    \centering
    \begin{tikzpicture}
    \begin{axis}[
            smooth,
            thick,
            xmode=log,
            legend pos=south east,
            xlabel=K,
            ylabel=Accuracy@K,
            ytick={0.0,0.1,0.2,0.3,0.4,0.5,0.6,0.7,0.8,0.9,1.0},
            log ticks with fixed point,
            minor tick num=5,
            grid=both,
            grid style={line width=0.2pt, draw=gray!20},
            major grid style={line width=0.5pt,draw=gray!40},
            skip coords between index={2}{4},
            skip coords between index={5}{9},
            skip coords between index={10}{19},
            skip coords between index={20}{39},
            skip coords between index={40}{99},
            skip coords between index={100}{199},
            xmajorgrids,
            ymajorgrids,
            width=1.0\columnwidth,
            height=0.8\columnwidth
        ]
        \addplot table [x=k, y=r1-f, col sep=comma] {data/accuracy-at-k_nq-open-test.csv};
        \addplot table [x=k, y=r2-f, col sep=comma] {data/accuracy-at-k_nq-open-test.csv};
        \addplot+[mark=triangle*] table [x=topk, y=selection, col sep=comma] {data/accuracy-at-k_reader-full_nq-open-test.csv};
        
        \legend{
            retrieved, 
            reranked, 
            reader
        };
    \end{axis}%
\end{tikzpicture}%
%
    \caption{Accuracy@K on test-data of NQ-Open.}%
    \label{fig:acc_k_nq}
\end{figure}%

\subsection{Extractive Reader Performance}
\label{sec:ext_reader_perf}
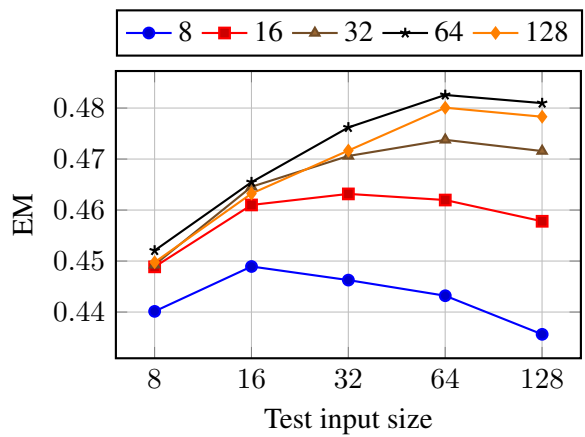
\begin{figure}[t]
    \centering
    \begin{tikzpicture}
    \begin{axis}[
            thick,
            xmode=log,
            ytick={0.44,0.45,0.46,0.47,0.48},
            log ticks with fixed point,
            log basis x=2,
            legend columns=-1,
            legend style={
                at={(0,1.05)}, 
                /tikz/every even column/.append style={column sep=1pt}, 
                legend cell align={left},  
                anchor=south west
            },
            xlabel=Test input size,
            ylabel=EM,
            xmajorgrids,
            ymajorgrids,
            width=1\columnwidth,
            height=0.7\columnwidth,
        ]
        \addplot table [x=test_input_size, y=8, col sep=comma] {data/extractive-reader-batch-sizes.csv};
        \addplot table [x=test_input_size, y=16, col sep=comma] {data/extractive-reader-batch-sizes.csv};
        \addplot+[mark=triangle*] table [x=test_input_size, y=32, col sep=comma] {data/extractive-reader-batch-sizes.csv};
        \addplot table [x=test_input_size, y=64, col sep=comma] {data/extractive-reader-batch-sizes.csv};
        \addplot+[color=orange,mark options={fill=orange}] table [x=test_input_size, y=128, col sep=comma] {data/extractive-reader-batch-sizes.csv};

        \addlegendentry{8}
        \addlegendentry{16}
        \addlegendentry{32}
        \addlegendentry{64}
        \addlegendentry{128}
    \end{axis}%
\end{tikzpicture}%
%
    \caption{Influence of test input size on extractive reader's performance for various train input sizes (different curves) on NQ-Open test dataset.}%
    \label{fig:ext_reader_batch_analysis}
\end{figure}%
In order to investigate the influence of the number of input passages on the extractive reader's performance, we trained multiple ELECTRA-base models, each with different input size. In test time, we evaluate each of them on various input sizes. Figure \ref{fig:ext_reader_batch_analysis} shows that increasing train/test input size has a positive influence on extractive reader's performance. However, input size 128 doesn't seem to increase the performance anymore.


\subsection{Ablations}

The ablations are listed in Table~\ref{tab:ablation_study}. 
We ablate results without using passage reranker, with separate readers and their combination and with different stages of component fusion. 
Namely, performing a \emph{naive} answer re-ranking by generative reader means the system chooses the most probable answer span among the top-$M$ spans provided by the extractive reader according to generative reader log-probabilities as shown in equation \eqref{eq:genrerank}. 
Analogously, the \emph{aggr} fusion denotes that the system chooses the most probable answer span according to aggregated scores, as in equation \eqref{eq:aggloss}. 
Finally, the \emph{aggr+bd} fusion denotes the binary decision, as shown in equation \eqref{eq:bdformula}. 

\begin{table}[t]
    \centering
    \scalebox{1.0}{\begin{tabular}{cccccc}
\cmidrule[\heavyrulewidth]{2-6}
& $\boldsymbol{P}_*$ & \multicolumn{1}{c}{$\emptyset$} & \multicolumn{1}{c}{\textbf{$\{r\}$}} & \multicolumn{1}{c}{\textbf{$\{rr\}$}} & \multicolumn{1}{c}{\textbf{$\{r,rr\}$}} \\
\cmidrule{2-6}
\multirow{3}{*}{\rotatebox[origin=c]{90}{\small NQ-Open}}      & $\{e\}$       & 50.72                           & 51.41                                & 51.55                                 & 51.69                                   \\
                       & $\{g\}$       & 52.44                           & 52.88                                & 53.35                                 & 53.19                                   \\
                       & $\{e,g\}$     & 54.63                           & \textbf{55.10}                       & 54.82                                 & 54.90                                   \\ \cmidrule{2-6}
\multirow{3}{*}{\rotatebox[origin=c]{90}{\small TQ-Open}}      & $\{e\}$       & 65.54                           & 65.64                                & 65.60                                 & 65.61                                   \\
                       & $\{g\}$       & 68.25 & 68.17 & 68.21 & 68.26 \\
                       & $\{e,g\}$     & 68.45 & 68.57 & \textbf{68.66} & \textbf{68.66} \\ \cmidrule[\heavyrulewidth]{2-6}
\end{tabular}}
    \caption{Results for different pipeline components used for score aggregation on NQ-Open a TQ-Open. See text for details.}
    \label{tab:score-aggr}
\end{table}
As expected, we observe that reranker improves the results consistently for generative model in all cases. 
The gains are especially large for TQ-Open (over 3.7 EM, underscored in Table~\ref{tab:ablation_study}). 
In fact, the results are comparable to \citet{izacard2020leveraging}, suggesting that using the FiD reader with smaller context window and reranker is a reasonable alternative to memory inefficient FiD with large input size.
Furthermore as expected, the extractive reader without reranker already has top-128 passages at the input, and improvements from the passage reranking are only negligible if any (less than 1 EM).

Finally, the results on NQ-Open and EfficientQA suggest applying the binary decision does not bring large improvements over the score aggregation if any. 
However, notice that this is not the case for TQ-Open, where the generative reader performs significantly better compared to extractive reader, suggesting both component fusions play important role in the system.

\begin{table}[t]
    \centering
    \scalebox{1.0}{\begin{tabular}{@{}cccccc}

\cmidrule[\heavyrulewidth]{2-6}
& $\boldsymbol{P}_*$ & \multicolumn{1}{c}{$\emptyset$} & \multicolumn{1}{c}{$\{r\}$} & \multicolumn{1}{c}{$\{rr\}$} & \multicolumn{1}{c}{$\{r,rr\}$} \\ \cmidrule{2-6}
\multirow{3}{*}{\rotatebox[origin=c]{90}{\small NQ-Open}} & $\{e\}$   & 52.85 & 53.30          & 53.10 & 52.94 \\
                                                          & $\{g\}$   & 52.44 & 52.77          & 53.21 & 53.07 \\
                                                          & $\{e,g\}$ & 54.35 & \textbf{55.10} & 54.46 & 54.99 \\ \cmidrule{2-6}
\multirow{3}{*}{\rotatebox[origin=c]{90}{\small TQ-Open}} & $\{e\}$   & 69.34 & 69.28 & 69.23 & 69.26 \\
                                                          & $\{g\}$   & 69.76 & 69.71 & 69.65 & 69.77 \\
                                                          & $\{e,g\}$ & 69.80 & 69.89 & 69.88 & \textbf{69.94} \\ 
\cmidrule[\heavyrulewidth]{2-6}
\end{tabular}}
    \caption{Results for binary decision on NQ-Open and TQ-Open for different aggregated pipeline components from Table \ref{tab:score-aggr}.}
    \label{tab:binary-decision}
\end{table}


\subsection{Component Fusion}
\label{sec:component_fusion}
Furthermore, we analyze the performance of each component combination in the score aggregation and its impact on the component fusion via binary decision.
Both fusions are tuned on validation data and reported on test data of the NQ-Open and TQ-Open datasets. See Appendix \ref{app:additional_comp_fusion} for analysis on additional datasets.
Table \ref{tab:score-aggr} shows all relevant combinations of ranker \emph{r}, reranker \emph{rr}, extractive reader \emph{e} and generative reader \emph{g} probabilities used in score aggregation. 
In overall, we observe that combining retriever and reranker scores with the reader leads to better or equal performance. On NQ-Open, we observe minor improvements up to \textasciitilde1 EM. However, there is no difference on TQ-Open. 

The impact of adding a binary decision after the score aggregation is shown in Table~\ref{tab:binary-decision}. 
Interestingly, the binary decision component significantly improves the performance only without reranked answer scores ($\{e\}$ rows in both tables). 
However, fusing the generative and extractive reader via binary decision performs significantly worse on NQ-Open than fusing both readers together with score aggregation ($\{e\}$ row in Table~\ref{tab:binary-decision} vs. $\{e,g\}$ row in Table~\ref{tab:score-aggr}). As already noted in ablations, we find this to be quite the opposite for TQ-Open. We hypothesize that the binary decision is strong in cases, where generative reader performs better to extractive reader (the case of TQ-Open). We argue that if the generative reader is better, the abstractive answer should be used far more often, than when it's not.
We support the hypothesis by analyzing the proportion of test samples, on which the binary decision component activated (i.e. an abstractive prediction was selected). On NQ-Open, the component almost never activated (only on 3.5\% samples), but this proportion is much higher (26.6\%) on TQ-Open. 

\begin{table}[t]
    \centering
    \scalebox{1.0}{\begin{tabular}{ccccc}
\toprule
\textbf{Readers}      & \textbf{Ensemble} & \textbf{EM}    & \textbf{$\Delta_{ext}$} & \textbf{$\Delta_{gen}$} \\ \midrule
\multirow{3}{*}{ext} & -                 & 46.79          & -                       & -                      \\
                     & 2 models        & 48.30          & 1.51                    & -                      \\
                     & 3 models        & 48.59          & 1.80                    & -                      \\ \midrule
\multirow{3}{*}{gen} & -                 & 45.00          & -                       & -                      \\
                     & 2 models        & 46.30          & -                       & 1.30                   \\
                     & 3 models        & 46.59          & -                       & 1.59                   \\ \midrule
ext+gen              & aggr           & \textbf{49.92} & \textbf{3.13}           & \textbf{4.92}    
\\ \bottomrule
\end{tabular}}%
    \caption{Comparison between ensembling via posterior averaging and score aggregation on NQ-Open.}
    \label{fig:reader_ensemble}
\end{table}
\subsection{Comparison with Posterior Averaging}
Finally, we compare our score aggregation with the ensemble computed via posterior probability averaging. In particular, we train three extractive and generative base-sized models initialized with different random seed. We do not use reranker in this experiment, and set train/test input size of extractive reader to 32. We assess the predictions using the averaged posterior probabilities and compare their average performance with score aggregation in Table \ref{fig:reader_ensemble}. Concretely, we compare with average of all 2 model ensembles (2 models) and with an ensemble of all 3 checkpoints (3 models). We observe two to three times improvement of score aggregation over the posterior probability averaging on NQ-Open test data.


\section{Related Work}
\begin{description}[style=unboxed,leftmargin=0em,listparindent=\parindent,parsep=0pt,]


\item[Passage reranking.]
Previous work in QA based on neural nets used Bi-LSTM encoders \cite{wang2018r3,lee2018ranking} that score each document independently. Over time, Bi-LSTM were replaced by BERT-like transformer encoders \cite{qiao2019understanding,wang2019multi}. 
 For document ranking, \citet{nogueira2019multistage} proposed a multi-stage architecture. The first stage scores each document independently, and the second estimates the more relevant document from all document pairs. Another document ranking approach uses the seq2seq model to generate a true or false answer to the document's relevance to the query \cite{nogueira2020document}. 
Recent works have often focused on effective reranking. \citet{xin2020early} achieved inference speedup using early exiting, \citet{jang2020document} proposed a smaller and faster model, and \citet{mao2021reader} came up with a method which uses reader's predictions to rerank the passages.
Our reranker is most similar to \citet{nogueira2019passage, luan2020sparse}, except that unlike in IR, we assume there is just one correct passage and thus train our model via categorical cross-entropy.

\item[Multipassage Reading Comprehension] 
Related work considers generative and extractive approaches towards modeling the reader. The generative reader generates an answer while conditioned on question alone \cite{roberts2020much}, or question with relevant passages \cite{lewis2020retrieval,min2020ambigqa}. \citet{izacard2020leveraging} showed it suffices to concatenate the passages in the decoder of seq2seq model, increasing the amount of top-passages the model can depend on dramatically.
The extractive reader used in Open-QA assumes that the answer is a continuous span string in located in retrieved paragraphs  \cite{chen2017reading}. \citet{clark-gardner-2018-simple} proposed to aggregate the probabilities of distantly supervised answer matches via maximum marginal likelihood (MML). \citet{lin2018denoising} proposed to denoise distantly supervised answer string matches in MML via paragraph-ranker. 
\citet{cheng2020probabilistic} experimented with different assumptions for MML, showing improvement when marginalizing over components of span probability independently. \citet{fajcik2020rethinking} proposed to model joint span probability directly via compound objective, instead of modeling the probability of span's start and end independently. \citet{karpukhin2020dense} incorporated an independent passage classifier loss to his MML objective. Our objective is similar to the last work, except that it uses joint component and also optimizes MML over relevant passages' probabilities.

\item[Component Fusion.]
 \citet{yang-etal-2019-end} also combined BM25 ranker and reader scores via linear combination. Our work can be seen as an extension of this idea to combining the scores of all pipeline's components.
\citet{iyerreconsider2020} proposed a system which directly learns to rerank question-passage-answer triplets proposed via extractive model. However, reranking answers from their large extractive model via large reranker leads to \textasciitilde1 EM improvement absolute, whereas R2-D2s score aggregation improves 4 to 5 EM w.r.t. the extractive reader. Concurrently with our work, \citet{cheng2021unitedqa} proposed hard voting ensembling scheme to combine the reader predictions. Firstly, each model from an ensemble produces its best prediction, then the votes for identical predictions are combined, omitting the scores produced by the individual models. The authors obtained best results using two FiD readers and single extractive reader, leading to 1.6 and 2.4 EM improvement on TQ-Open and NQ-Open, compared to their best single extractive or generative model.
 
\end{description}

\section{Conclusion}
This work proposed R2-D2, a novel state-of-the-art pipeline for open-domain QA based on 4 components: retriever, reranker, generative reader and extractive reader. 
We showed that employing a reranker is a reasonable alternative to using large passage counts at the input of both the extractive and the generative reader. 
Our results on NQ-Open and EfficientQA showed that the extractive and the generative reader could perform equally in Open-QA, although the generative reader is twice the size of the extractive reader. 
On the other hand, we observe the extractive reader underperforms on TQ-Open. We hypothesize, that the cause is (1) the complexity of trivia questions with many entities, which often require combining evidence from multiple passages --- these are impossible to answer for the extractive reader by design --- and (2) the expensive hyperparameter search, as we used NQ-Open hyperparameters also for TQ-Open.
Contrary to belief based on the results on different datasets \cite{yang-etal-2019-end,wang-etal-2019-multi,izacard2020leveraging}, we found the extractive reader can also benefit from larger input sizes, both in training and test time. 
Finally, we proposed a component fusion, which allows merging the complementary behavior of generative and extractive approaches along with the ranking components and found it improves the results significantly. 
Due to its heterogenous and modular nature, our pipeline forms an ideal base for future research of component integration in modern Open-QA. 

%
%





\section*{Acknowledgments}
We would like to thank Jan Doležal for implementing an R2-D2 demo.
This work was supported by the Czech Ministry of Education, Youth and Sports, subprogram INTERCOST, project code: LTC18054.
The computation used the infrastructure supported by the Ministry of Education, Youth and Sports of the Czech Republic through the e-INFRA CZ (ID:90140).

\bibliography{anthology,custom}
\bibliographystyle{acl_natbib}

\appendix
\clearpage


\section{Additional Accuracy@K Analysis}
\label{app:accuracy_at_k}
Analysis of Accuracy@K on NQ-Open development data in Figure \ref{fig:accuracy-at-k_nq-open-dev}, on EfficientQA data is shown in Figure \ref{fig:accuracy-at-k_efficientqa}, and on TQ-Open development data in Figure \ref{fig:accuracy-at-k_trivia-dev} and test data in Figure \ref{fig:accuracy-at-k_trivia-test}. 

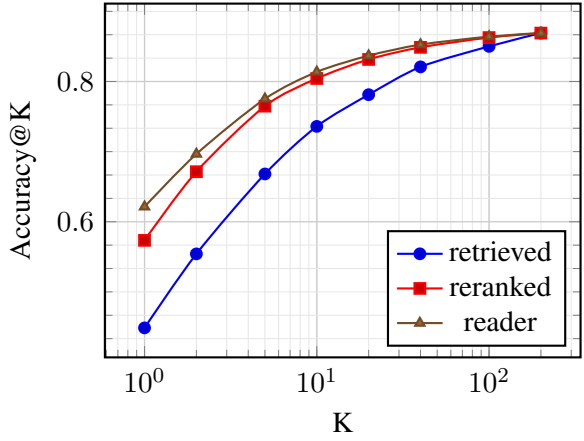
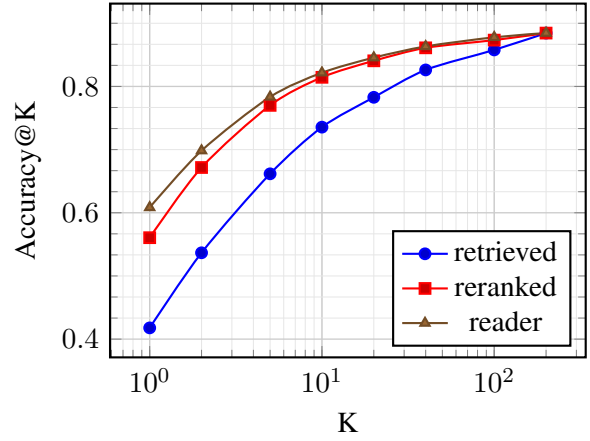
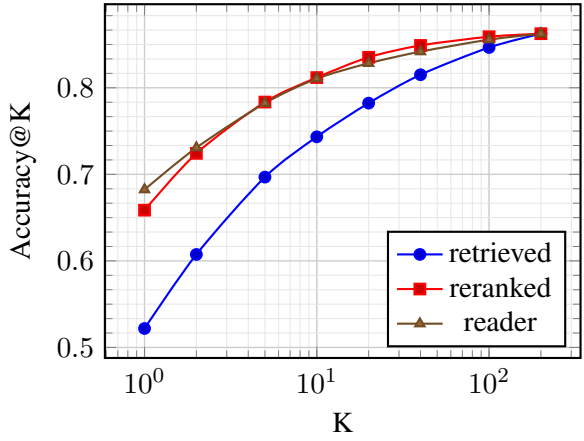
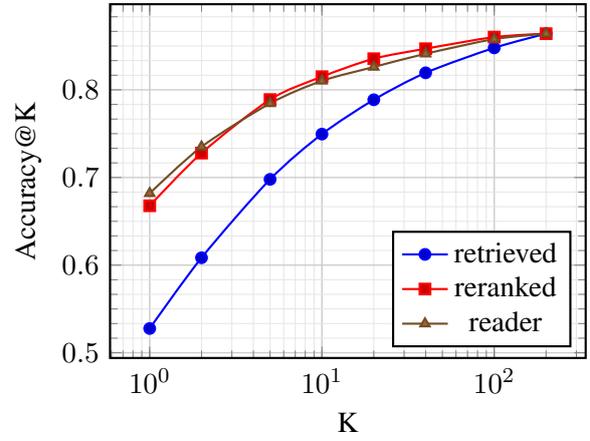
\begin{figure*}[th]
\begin{subfigure}{0.49\textwidth}
    \centering
    \scalebox{1.00}{\begin{tikzpicture}
    \begin{axis}[
            smooth,
            thick,
            xmode=log,
            legend pos=south east,
            xlabel=K,
            ylabel=Accuracy@K,
            skip coords between index={2}{4},
            skip coords between index={5}{9},
            skip coords between index={10}{19},
            skip coords between index={20}{39},
            skip coords between index={40}{99},
            skip coords between index={100}{199},
            minor tick num=5,
            grid=both,
            grid style={line width=0.2pt, draw=gray!20},
            major grid style={line width=0.5pt,draw=gray!40},
            xmajorgrids,
            ymajorgrids,
            width=1\textwidth,
            height=0.8\textwidth,
        ]
        \addplot table [x=k, y=r1-f, col sep=comma] {data/accuracy-at-k_nq-open-dev.csv};
        \addplot table [x=k, y=r2-f, col sep=comma] {data/accuracy-at-k_nq-open-dev.csv};
        \addplot+[mark=triangle*] table [x=topk, y=selection, col sep=comma] {data/accuracy-at-k_reader-full_nq-open-dev.csv};
        
        \legend{
            retrieved, 
            reranked, 
            reader
        };
    \end{axis}%
\end{tikzpicture}%
}
    \caption{NQ-Open (dev).}
    \label{fig:accuracy-at-k_nq-open-dev}
\end{subfigure}
\hfill
\begin{subfigure}{0.49\textwidth}
    \centering
    \scalebox{1.00}{\begin{tikzpicture}
    \begin{axis}[
            smooth,
            thick,
            xmode=log,
            legend pos=south east,
            xlabel=K,
            ylabel=Accuracy@K,
            skip coords between index={2}{4},
            skip coords between index={5}{9},
            skip coords between index={10}{19},
            skip coords between index={20}{39},
            skip coords between index={40}{99},
            skip coords between index={100}{199},
            minor tick num=5,
            grid=both,
            grid style={line width=0.2pt, draw=gray!20},
            major grid style={line width=0.5pt,draw=gray!40},
            xmajorgrids,
            ymajorgrids,
            width=1\textwidth,
            height=0.8\textwidth,
        ]
        \addplot table [x=k, y=r1-f, col sep=comma] {data/accuracy-at-k_efficientqa.csv};
        \addplot table [x=k, y=r2-f, col sep=comma] {data/accuracy-at-k_efficientqa.csv};
        \addplot+[mark=triangle*]  table [x=topk, y=selection, col sep=comma] {data/accuracy-at-k_reader-full_efficientqa.csv};
                
        \legend{
            retrieved, 
            reranked, 
            reader
        };
    \end{axis}%
\end{tikzpicture}%
}
    \caption{EfficientQA.}
    \label{fig:accuracy-at-k_efficientqa}
\end{subfigure}
\par\bigskip
\begin{subfigure}{0.49\textwidth}
    \centering
    \scalebox{1.00}{\begin{tikzpicture}
    \begin{axis}[
            smooth,
            thick,
            xmode=log,
            legend pos=south east,
            xlabel=K,
            ylabel=Accuracy@K,
            skip coords between index={2}{4},
            skip coords between index={5}{9},
            skip coords between index={10}{19},
            skip coords between index={20}{39},
            skip coords between index={40}{99},
            skip coords between index={100}{199},
            minor tick num=5,
            grid=both,
            grid style={line width=0.2pt, draw=gray!20},
            major grid style={line width=0.5pt,draw=gray!40},
            xmajorgrids,
            ymajorgrids,
            width=1\textwidth,
            height=0.8\textwidth,
        ]
        \addplot table [x=k, y=r1-f, col sep=comma] {data/accuracy-at-k_trivia-dev.csv};
        \addplot table [x=k, y=r2-f, col sep=comma] {data/accuracy-at-k_trivia-dev.csv};
        \addplot+[mark=triangle*] table [x=topk, y=selection, col sep=comma] {data/accuracy-at-k_reader-full_trivia-dev.csv};
        
        \legend{
            retrieved, 
            reranked, 
            reader
        };
    \end{axis}%
\end{tikzpicture}%
}
    \caption{TQ-Open (dev).}
    \label{fig:accuracy-at-k_trivia-dev}
\end{subfigure}
\hfill
\begin{subfigure}{0.49\textwidth}
    \centering
    \scalebox{1.00}{\begin{tikzpicture}
    \begin{axis}[
            smooth,
            thick,
            xmode=log,
            legend pos=south east,
            xlabel=K,
            ylabel=Accuracy@K,
            skip coords between index={2}{4},
            skip coords between index={5}{9},
            skip coords between index={10}{19},
            skip coords between index={20}{39},
            skip coords between index={40}{99},
            skip coords between index={100}{199},
            minor tick num=5,
            grid=both,
            grid style={line width=0.2pt, draw=gray!20},
            major grid style={line width=0.5pt,draw=gray!40},
            xmajorgrids,
            ymajorgrids,
            width=1\textwidth,
            height=0.8\textwidth,
        ]
        \addplot table [x=k, y=r1-f, col sep=comma] {data/accuracy-at-k_trivia-test.csv};
        \addplot table [x=k, y=r2-f, col sep=comma] {data/accuracy-at-k_trivia-test.csv};
        \addplot+[mark=triangle*] table [x=topk, y=selection, col sep=comma] {data/accuracy-at-k_reader-full_trivia-test.csv};

        \legend{
            retrieved, 
            reranked, 
            reader
        };
    \end{axis}%
\end{tikzpicture}%
}
    \caption{TQ-Open (test).}
    \label{fig:accuracy-at-k_trivia-test}
\end{subfigure}
\caption{Analysis of Accuracy@K on different datasets.}
\end{figure*}
\section{Additional Component Fusion Analysis}
\label{app:additional_comp_fusion}
This section includes results analogical to Tables \ref{tab:score-aggr}, \ref{tab:binary-decision} on EfficientQA and development data of NQ-Open and TQ-Open (Tables \ref{tab:score-aggr_appendix}, \ref{tab:binary-decision_appendix}).

\begin{table}[H]
    \centering
    \scalebox{1.0}{\begin{tabular}{cccccc}

\cmidrule[\heavyrulewidth]{2-6}
& $\boldsymbol{P}_*$ & \multicolumn{1}{c}{$\emptyset$} & \multicolumn{1}{c}{$\{r\}$} & \multicolumn{1}{c}{$\{rr\}$} & \multicolumn{1}{c}{$\{r,rr\}$} \\ \cmidrule{2-6}
\multirow{3}{*}{\rotatebox[origin=c]{90}{\small NQ-Open}} & $\{e\}$   & 48.38 & 48.94 & 48.85 & 49.14 \\
                                                          & $\{g\}$   & 49.99 & 50.49 & 50.35 & 50.47 \\
                                                          & $\{e,g\}$ & 51.79 & \textbf{51.97} & 51.82 & 51.80 \\
\cmidrule{2-6}
\multirow{3}{*}{\rotatebox[origin=c]{90}{\small TQ-Open}} & $\{e\}$   & 65.07 & 65.21 & 65.16 & 65.24 \\
                                                          & $\{g\}$   & 67.68 & 67.72 & 67.73 & 67.76 \\
                                                          & $\{e,g\}$ & 68.13 & \textbf{68.19} & 68.17 & 68.12 \\
\cmidrule{2-6}
\multirow{3}{*}{\rotatebox[origin=c]{90}{\small EfficientQA}} & $\{e\}$   & 47.56 & 48.33 & 48.89 & 48.72          \\
                                                              & $\{g\}$   & 49.11 & 49.56 & 50.22 & 50.11          \\
                                                              & $\{e,g\}$ & 50.78 & 51.67 & 50.89 & \textbf{52.00} \\
\cmidrule[\heavyrulewidth]{2-6}
\end{tabular}}
    \caption{Score aggregation on validation data of NQ-Open, TQ-Open and EfficientQA.}
    \label{tab:score-aggr_appendix}
\end{table}
\begin{table}[H]
    \centering
    \scalebox{1.0}{\begin{tabular}{cccccc}

\cmidrule[\heavyrulewidth]{2-6}
& $\boldsymbol{P}_*$ & \multicolumn{1}{c}{$\emptyset$} & \multicolumn{1}{c}{$\{r\}$} & \multicolumn{1}{c}{$\{rr\}$} & \multicolumn{1}{c}{$\{r,rr\}$} \\ \cmidrule{2-6}
\multirow{3}{*}{\rotatebox[origin=c]{90}{\small NQ-Open}} & $\{e\}$   & 50.65 & 51.24          & 51.01 & 51.17 \\
                                                          & $\{g\}$   & 50.36 & 50.91          & 50.68 & 50.90 \\
                                                          & $\{e,g\}$ & 52.24 & \textbf{52.29} & 52.27 & 52.07 \\
\cmidrule{2-6}
\multirow{3}{*}{\rotatebox[origin=c]{90}{\small TQ-Open}} & $\{e\}$   & 69.03 & 69.03 & 69.01 & 68.99 \\
                                                          & $\{g\}$   & 69.54 & 69.46 & 69.62 & 69.70 \\
                                                          & $\{e,g\}$ & 69.77 & \textbf{69.79} & 69.67 & 69.61 \\
\cmidrule{2-6}
\multirow{3}{*}{\rotatebox[origin=c]{90}{\small EfficientQA}} & $\{e\}$   & 48.33 & 50.06 & 49.39 & 49.67          \\
                                                          & $\{g\}$   & 48.94 & 49.50 & 50.06 & 49.72          \\
                                                          & $\{e,g\}$ & 50.78 & 51.83 & 50.94 & \textbf{52.22} \\
\cmidrule[\heavyrulewidth]{2-6}
\end{tabular}}
    \caption{Binary decision on NQ-Open, TQ-Open and EfficientQA.}
    \label{tab:binary-decision_appendix}
\end{table}

\section{Data Pre-processing}
\label{app:data_preprocessing}
This section describes how the training datasets for reranker and extractive reader are filtered, and how the distant supervision labeling is generated. Note not each example contains golden passage, as not each example can be mapped to the used dump of Wikipedia. We use the same golden passage mapping as \citet{karpukhin2020dense}.

For passage reranking, the input must contain at least one positive example. We meet this condition either by adding a golden passage or searching for the passage with an answer in the top-400 results retrieved by DPR. In detail about the search, first the Simple tokenizer proposed in DrQA\footnote{\url{https://github.com/facebookresearch/DrQA}} tokenizes each passage and golden answer. The positive example is the best-scored tokenized passage that contains an exact match with one of the tokenized answers. Note the search proceeds in the same way as in DPR's Accuracy@K implementation\footnote{\url{https://github.com/facebookresearch/DPR}}.

The extractive reader is trained only on samples which contain exact match to at least one of the annotated answers in the top-1 passage, or golden passage if it is available. 
The exact match is performed on the subword token level (i.e. in ELECTRA's tokenization).

Next, the span annotations are extracted from the passages at the reader's input. Note each sample may contain multiple answers. The annotations for each answer in each sample are obtained differently in retrieved passages and in the golden passage. 
For retrieved passages, we search for the answer's exact matches in passages, and use each match as target annotation. 
For golden passage, we also search for the answer's exact matches in it. If there is none, the answer is soft matched with single sub-sequence of golden passage, which yields highest non-zero F1 score. The F1 soft match is also performed on the subword token level. Therefore answers with zero highest F1 soft match with golden passage and no exact match in any of the reader's input passages are discarded.

\subsection{Upper Bound on F1 Matching}
Because the brute-force computation of a span with the greatest  nonzero F1 score is potentially very demanding, we found the length limit for spans that are worth searching (see Theorem \ref{the:F1worthSearchingTheorem}).

To compare brute-force with upper bound implementation, we run an experiment on 16,741 passages (retrieved for NQ-Open dev). The average time per passage for brute-force approach was 121 ms while it was only 9 ms for implementation that uses the upper bound optimization.

The soft match is described in Algorithm \ref{alg:softMatch}. It assumes that there is no exact match.

\begin{algorithm}
    \caption{Soft match}
    \label{alg:softMatch}
    
    \begin{algorithmic}[1]
    \Require set of spans S and answer span a
    \Function{SoftMatch}{$\texttt{S}, \texttt{a}$}
        \State $\texttt{actSize} \gets 1$
        \State $\texttt{lenLimit} \gets 2$
        \State $\texttt{bestSpan} \gets \texttt{None}$
        \State $\texttt{bestScore} \gets 0$
        
        \While{$\texttt{actSize} < \texttt{lenLimit}$}
            \ForAll{$t \in S$ of size $\texttt{actSize}$} 
                \State $\texttt{score} \gets \texttt{F1}(t,a)$
                \If{$\texttt{score}>\texttt{bestScore}$}
                    \State $\texttt{bestSpan} \gets t$
                    \State $\texttt{bestScore} \gets \texttt{score}$
                    \State $\texttt{lenLimit} \gets |a|\frac{|t|+|a|-s_{ta}}{s_{ta}}$
                \EndIf
            \EndFor
            \State $\texttt{actSize} \gets \texttt{actSize}+1$
        \EndWhile
    \State \Return $\texttt{bestSpan}$
    \EndFunction
    \end{algorithmic}
\end{algorithm}

\begin{lemma}
\label{lem:upBounIsGreater}
Let $t$ and $a$ be non-empty spans and $0<s_{ta}\leq|a|$ number of shared tokens for them. Then\footnote{|x| symbolises number of tokens in span x.} 
\begin{equation}
   |t| \leq |a|\frac{|t|+|a|-s_{ta}}{s_{ta}} \: .
\end{equation}
\end{lemma}
\begin{proof}

To prove it by contradiction assume that
\begin{equation}
    |t| > |a|\frac{|t|+|a|-s_{ta}}{s_{ta}} \:,
\end{equation}
then

\begin{equation}
    s_{ta}|t| > |a||t|+|a||a|-|a|s_{ta} \:,
\end{equation}
and also $0 < s_{ta} \leq |a|$, thus $|a||a|-|a|s_{ta} \geq 0$. Therefore even if we assume that\\$|a||a|-|a|s_{ta} = 0$. We get
\begin{equation}
\begin{split}
    s_{ta}|t| > |a||t| \\
    s_{ta} > |a| \: ,
\end{split}
\end{equation}

which is in contradiction with $0 < s_{ta} \leq |a|$.
\end{proof}

\begin{theorem}
\label{the:F1worthSearchingTheorem}
Let $S$ be a set of non-empty spans, $a$ an non-empty answer span, $t$ non-empty trial span, $0 < s_{ta} \leq|a|$ is number of shared tokens for $t$ and $a$, and $S_b = \{z | z \in S \land |z| \geq |a|\frac{|t|+|a|-s_{ta}}{s_{ta}}  \}$. Then the theorem states that 
\begin{equation}
    \forall x \in S_b(\operatorname{F1}(x,a) \leq \operatorname{F1}(t,a)) \: .
\end{equation}

\end{theorem}
\begin{proof}
To prove it by contradiction assume that

\begin{equation}
    \exists x \in S_b (\operatorname{F1}(x,a) > \operatorname{F1}(t,a)) \: .
\end{equation}

F1 score can be expressed as:
\begin{equation}
    \operatorname{F1(b,c)}=\frac{2s_{bc}}{|b|+|c|} \: ,
\end{equation}

thus 
\begin{equation}
\label{eq:ineqF1}
    \frac{2s_{xa}}{|x|+|a|} >  \frac{2s_{ta}}{|t|+|a|} \: .
\end{equation}

From Lemma \ref{lem:upBounIsGreater} $|t|\leq|x|$. Therefore $s_{ta}<s_{xa}$, to satisfy the inequality (in equation \ref{eq:ineqF1}), and we know that $0 < s_{xa} \leq |a|$. So let the $s_{xa}=|a|$ (the maximum) then

\begin{equation}
\begin{split}
    \frac{2|a|}{|x|+|a|} >  \frac{2s_{ta}}{|t|+|a|}
    \\
    |a|(|t|+|a|) > s_{ta}|x| + s_{ta}|a|
    \\
    |x| < |a| \frac{|t|+|a|-s_{ta}}{s_{ta}} \: ,
\end{split}
\end{equation}
which is in contradiction with $x \in S_b$.
\end{proof}

\section{Softmax Notation}
\label{app:softmax_not}
Usually, softmax function $ \sigma: \mathbb{R}^K \rightarrow \mathbb{R}^K $ is  defined as:
\begin{equation}
   \sigma(v)_i = \frac{e^{v_i}}{\sum^K_{j=1} e^{v_j}}.
\end{equation}

However, some parts of this work used variant of softmax that is defined as follows:
\begin{equation}
\begin{split}
    \operatorname*{softmax}\limits_{x \in D} \big({f(x)}\big)_{y} = \frac{e^{f(y)}}{\sum\limits_{x \in D} e^{f(x)}} \: ,
\end{split}
\end{equation}
where $D$ is the input set, $f : D \rightarrow \mathbb{R}$, $y \in D$.

\section{Decoding the Distributions from the Extractive Reader}
\label{app:decoding_ext_probs}
We analyzed the subsets of joint probability space over spans obtained via multiplication of distributions as explained in section \ref{ss:ext_reader} in Table \ref{tab:ext_r_prob_space}. The factors of this space are the distribution given by the outer product of independent probability distributions $\boldsymbol{P}_{start}(.) \boldsymbol{P}_{end}(.)^\top$ denoted as I, joint probability distribution $\boldsymbol{P}_{joint}(.)$ denoted as J, and passage distribution $\boldsymbol{P}_{passage}(.)$ denoted as C.
\begin{table}[H]
\scalebox{0.90}{
\begin{tabular}{cccc}
\toprule
\textbf{Factorization} & \multicolumn{1}{l}{\textbf{NQ-dev}} & \multicolumn{1}{l}{\textbf{NQ-test}} & \multicolumn{1}{l}{\textbf{EfficientQA}} \\ \midrule
I                                       & 48.32                              & 50.58                               & 47.33                                   \\
J                                       & 48.53                              & \textbf{51.25}                      & \textbf{47.83}                          \\
I+J                                     & \textbf{48.57}                     & 50.83                               & \textbf{47.83}                          \\
I+C                                     & 48.22                              & 50.55                               & 47.22                                   \\
J+C                                     & 48.49                              & 51.11                               & 47.56                                   \\
I+J+C                                   & 48.50                                & 50.86                               & 47.67                                   \\\bottomrule
\end{tabular}
}

    \caption{The results of the pipeline with different types of extractive reader's distribution used for decoding. See text for details.}
    \label{tab:ext_r_prob_space}
\end{table}

\begin{table}[H]
    \centering
    \begin{tabular}{cccc}
\toprule
\begin{tabular}[c]{@{}c@{}}\textbf{marginalizes} \\ \textbf{independently}\end{tabular}   & \begin{tabular}[c]{@{}c@{}}\textbf{joint} \\ \textbf{comp.}\end{tabular}     & \begin{tabular}[c]{@{}c@{}}\textbf{start\&end} \\ \textbf{comp.}\end{tabular} & \textbf{EM} \\ \midrule
-             & -              & \checkmark   & 45.42       \\
-             & \checkmark     & \checkmark   & 45.41       \\
\checkmark    & -              & \checkmark   & 45.71       \\
\checkmark    & \checkmark     & -            & \textbf{47.09}       \\
\checkmark    & \checkmark     & \checkmark   & 47.06       \\
\bottomrule
\end{tabular}
    \caption{Ablation of loss components on NQ-Open test dataset using ELECTRA-base model.}
    \label{tab:ext_r_ablation}
\end{table}

\section{Passage Reranker Revision}
In preliminary experiments of this work we used a Longformer encoder \cite{Beltagy2020Longformer} with concatenated passages at it's input to benefit from the early fusion between passages. In particular, the passages at the Longformer's input were shuffled and concatenated, and we used presoftmax score computed from the first Longformer's output representation of each passage as the $rerank$ function. The passages were shuffled with a fixed seed in both, training and test time. Therefore each passage was scored not only according to the question but also according to other passages. However, we did not observe any significant benefits when we used the Longformer setup over a RoBERTa which scores each passage independently (see Table~\ref{tab:old_ablation_study}). 

\begin{table*}
    \centering
    \scalebox{0.78}{\begin{tabular}{cc|rrr|rrr|rrr|rrr}
\toprule
 \multirow{2}{*}{\textbf{Readers}} & \multirow{2}{*}{\textbf{Fusion}} & \multicolumn{3}{c|}{\textbf{NQ-Open (dev)}}                      & \multicolumn{3}{c|}{\textbf{NQ-Open (test)}}                                         & \multicolumn{3}{c|}{\textbf{TQ-Open (test)}}                                         & \multicolumn{3}{c}{\textbf{EfficientQA}}                                           \\
                                   &                                  & Long. & \multicolumn{1}{c}{RoB.} & \multicolumn{1}{c|}{$\Delta$} & \multicolumn{1}{c}{Long.} & \multicolumn{1}{c}{RoB.} & \multicolumn{1}{c|}{$\Delta$} & \multicolumn{1}{c}{Long.} & \multicolumn{1}{c}{RoB.} & \multicolumn{1}{c|}{$\Delta$} & \multicolumn{1}{c}{Long.} & \multicolumn{1}{c}{RoB.} & \multicolumn{1}{c}{$\Delta$} \\ \midrule
 ext                               & -                                & 48.50 & 48.38                    & -0.12                         & 50.86                     & 50.72                    & -0.14                         & 65.41                     & 65.46                    & -0.05                         & 47.67                     & 47.56                    & -0.11                        \\
 gen                               & -                                & 49.34 & 49.40                    & 0.06                          & 51.50                     & 50.69                    & -0.81                         & 68.85                     & 69.14                    & -0.29                         & 47.33                     & 47.33                    & 0.00                         \\
ext+gen                           & naive                            & 49.91 & 49.99                    & 0.08                          & 53.43                     & 52.44                    & -0.99                         & 67.82                     & 68.01                    & -0.19                         & 49.06                     & 49.11                    & 0.05                         \\
ext+gen                           & aggr                             & 52.05 & 51.80                    & -0.25                         & 54.96                     & 54.90                    & -0.06                         & 68.49                     & 68.66                    & -0.17                         & 51.56                     & 52.00                    & 0.44                         \\
ext+gen                           & aggr+bd                          & 52.36 & 52.07                    & -0.29                         & 55.01                     & 54.99                    & -0.02                         & 69.62                     & 69.94                    & -0.32                         & 51.06                     & 52.22                    & 1.16                         \\ \bottomrule
\end{tabular}}%
    \caption{Exact match comparison of Longformer (Long.) and RoBERTa (RoB.) based passage reranker.}
    \label{tab:old_ablation_study}
\end{table*}

\section{Ablating the Extractive Reader's Objective} 
\label{app:ext_r_ablations}

Firstly let us demonstrate that start and end components in the used loss (see equation \ref{eq:extReaderIndLoss}) perform summation over both inter-passage and intra-passage combinations of starts and ends:
\begin{equation} \label{eq:extReaderIndStartEndComponentsComb}
\begin{split}
	-\log \sum_{s \in S} \boldsymbol{P}_{start}(s)
	-\log \sum_{e \in E} \boldsymbol{P}_{end}(e) = \\ =
	-\log \sum_{s \in S} \boldsymbol{P}_{start}(s) \sum_{e \in E} \boldsymbol{P}_{end}(e) = \\ =
	-\log \sum_{s \in S, e \in E} \boldsymbol{P}_{start}(s) \boldsymbol{P}_{end}(e) \, .
\end{split}
\end{equation}

Where $S=starts(C_{rr})$, $E=ends(C_{rr})$ and distribution dependencies are dropped for clarity.
Inter-passage combinations obviously do not correspond to a real answer. Even though that this loss does not reflect the task correctly, it achieves better results (see Table \ref{tab:ext_r_ablation}) than the following loss 

\begin{equation} \label{eq:extReaderJoiLoss}
     -\log \sum_{c \in C_{rr}} \sum_{a_e \in \operatorname{answers}(c)} \boldsymbol{P}_{e}(a_e|q, \mathcal{C}_{rr})
\end{equation}

that marginalizes components jointly, and thus the summation is done only through intra-passage start-end combinations. Such results agree with previous work \cite{cheng-etal-2020-probabilistic}.

Table \ref{tab:ext_r_ablation} also shows that the joint component improves the independent loss variant, but not the other one that marginalizes jointly. We hypothesize that this is because the loss in equation \ref{eq:extReaderJoiLoss} already considers only the intra-passage start-end pairs. Lastly, Table \ref{tab:ext_r_ablation} shows that using just the joint and passage component is sufficient for NQ-Open, which agrees with \citet{fajcik2020rethinking}. 


\section{Results According to Question and Answer Test-Train Overlap}
\label{app:train_test_overlap}
In addition to evaluation on the TQ-Open and NQ-Open shown in Table \ref{tab:systems}, we also report results on subsets of these datasets in Table \ref{tab:results_overlap_subsets}, as split by \citet{lewis-etal-2021-question}. We compare  R2-D1 (retriever, reranker and extractive or generative reader, marked as gen and ext respectively) and R2-D2 (ext+gen) to official results on FiD \cite{izacard2020leveraging}.

\begin{table*}
    \centering
    \begin{tabular}{c|rrrr|rrrr}
\toprule
\multicolumn{1}{c|}{\multirow{1}{*}{\bf Model}}                   & \multicolumn{4}{c|}{\bf NQ-Open}                                                                                                                                                                                                                                                                                   & \multicolumn{4}{c}{\bf TQ-Open}                                                                                                                                                                                                                                                                                                                         \\
\multicolumn{1}{l|}{} & \multicolumn{1}{c}{\small Total} & \multicolumn{1}{c}{\small\begin{tabular}[c]{@{}c@{}}Question\\ Overlap\end{tabular}} & \multicolumn{1}{c}{\small\begin{tabular}[c]{@{}c@{}} Answer \\ Overlap \\ Only\end{tabular}} & \multicolumn{1}{c|}{\small\begin{tabular}[c]{@{}c@{}} No\\ Overlap\end{tabular}} & \multicolumn{1}{c}{\small Total} & \multicolumn{1}{c}{\small\begin{tabular}[c]{@{}c@{}}Question\\ Overlap\end{tabular}} & \multicolumn{1}{c}{\small\begin{tabular}[c]{@{}c@{}}Answer\\ Overlap\\ Only\end{tabular}} & \multicolumn{1}{c}{\small\begin{tabular}[c]{@{}c@{}}No\\ Overlap\end{tabular}} \\ \midrule
FiD                   & 51.40                            & 71.30                                                  & 48.30                                                                                                    & 34.50                                                                                                   & 67.60                            & 87.50                                                                                        & 66.90                                                                                                    & 42.80                                                                                                  \\ \midrule
ext+gen               & 54.99                            & 75.00                                                  & 48.89                                                                                                    & 39.91                                                                                                   & 69.94                            & 90.18                                                                                        & 71.53                                                                                                    & 44.83                                                                                                  \\
gen                   & 50.69                            & 70.06                                                  & 46.98                                                                                                    & 34.04                                                                                                   & 69.14                            & 87.50                                                                                        & 70.32                                                                                                    & 44.83                                                                                                  \\
ext                   & 50.72                            & 72.53                                                  & 45.40                                                                                                    & 35.11                                                                                                   & 65.46                            & 83.63                                                                                        & 66.42                                                                                                    & 39.46                                                                                                  \\ \midrule
$\Delta_{\text{gen}}$ & 4.30 & 4.94 & 1.91 & 5.87 & 0.80 & 2.68 & 1.21 & 0.00 \\
$\Delta_{\text{ext}}$ & 4.27 & 2.47 & 3.49 & 4.80 & 4.48 & 6.55 & 5.11 & 5.37 \\
\bottomrule
\end{tabular}
    \caption{Results on the overlapping and non-overlapping parts of test sets for NQ and TQ. \textit{Total} column corresponds to overall result on the whole dataset, as reported before,  \textit{Question Overlap} corresponds to samples with train-test question overlap and answer overlap, \textit{Answer Overlap Only} corresponds to samples with answer overlap, but no question overlap, and \textit{No Overlap} corresponds to samples with no overlap between train and test sets.}
    \label{tab:results_overlap_subsets}
\end{table*}

\section{Inference Speed of Our Implementation}
\label{app:inference time}
While optimizing the R2-D2’s inference speed was not the main focus of this paper, we show that even our unoptimized implementation can be used in practice in small scale. 
We analyze the speed of our implementation on NQ-Open test data in Table~\ref{tab:inference_time}. 
The times were measured on a workstation with Intel Xeon Silver 4214 48-core CPU, 188GB RAM and Nvidia 2080Ti 12GB GPU. 
Table columns  show settings with and without passage reranker. 
Table rows are split into two parts; \textit{intermediate} rows show time spent by the pipeline’s single component (e.g., row ext. reader shows what time the pipeline spent by running just the ext. reader), and \textit{total} rows show the total time taken by the whole pipeline. 
The retriever and reranker infer with batch sizes 32 and 100 respectively, the readers run with batch size 1.

We note that in retrieval, we do not use any approximate K-NN algorithm to facilitate retrieval of top-K nearest passages and instead do the dot product with the matrix of passages directly on the CPU.
Secondly, we note that we do not parallelize the inference of generative reader and extractive reader.
Thirdly, notice the difference in extractive reader’s speed with and without passage reranker is caused by its different input size (see details of extractive reader’s experiments setup in subsection \ref{ss:models_and_pipeline}). 
Finally, we compare the speed of our approach using FiD with 25 and 100 input passages, like in the original FiD implementation\footnote{We simply pass 100 input passages to the model trained with 25 passages in the experiment.}.
The ratios of our measurements are compared explicitly in Table \ref{tab:inference_ratios}. 

\begin{table*}
    \centering
    \begin{tabular}{@{}ccrr}
\cmidrule[\heavyrulewidth]{2-4}
& \textbf{Modules}                & \multicolumn{2}{c}{\textbf{Rankers}} \\
&                                 & \multicolumn{1}{c}{retriever} & \multicolumn{1}{c}{+reranker} \\
\cmidrule{2-4}
\multirow{7}{*}{\rotatebox[origin=c]{90}{\small intermediate}}
& retriever             & 0.21              & 0.21                      \\
& passage reranker      & -                 & 1.94                      \\
& ext. reader           & 2.21              & 0.35                      \\
& gen. reader (25)      & 0.55              & 0.55                      \\
& answer reranker (25)  &  3.11             & 3.11                      \\
& gen. reader (100)     & 1.85              & -                      \\
& answer reranker (100) & 11.67             & -                      \\
\cmidrule{2-4}
\multirow{5}{*}{\rotatebox[origin=c]{90}{\small total}}
& ext        & 2.41                      & 2.19                      \\
& gen  (25)   & 0.76               & 2.70                      \\
& ext+gen  (25)          & 6.08                & 6.16                      \\
& gen  (100)   & 2.06               & -                      \\
& ext+gen  (100)         & 15.94              & -                      \\
\cmidrule[\heavyrulewidth]{2-4}
\end{tabular}
    \caption{Inference times on NQ-Open in seconds per question. See text for details.}
    \label{tab:inference_time}
\end{table*}

\begin{table*}
    \centering
    \begin{tabular}{l|rrrr}
\toprule
\multicolumn{1}{c|}{\textbf{Setup ratios}}               & \multicolumn{4}{c}{\textbf{Modules}}          \\
                      & only gen. & ans. reranker & gen pipe. & ext+gen pipe. \\\midrule
gen(100) / gen(25)    &     3.36x &         3.75x &     2.71x &         2.62x \\
rr+gen(25) / gen(25)  &     $^{*}$1.00x &         $^{*}$1.00x &     3.55x &         1.01x \\
gen(100) / rr+gen(25) &     $^{*}$3.36x &         $^{*}$3.75x &     0.76x &         2.58x \\\bottomrule
\end{tabular}
    \caption{Ratios of inference times on NQ-Open. First two columns compare the speed in stage of generating abstractive answer (only gen.) and answer reranking (ans. reranker). The subsequent columns compare speed of whole pipeline just with generative reader and no component fusion (gen pipe.) and full R2-D2 pipeline (ext+gen pipe.). Row gen(100)/gen(25) compares the speedup of pipeline when using just 25 passages in FiD's input (denoted as gen(25)) instead of 100 (denoted as gen(100)). Row rr+gen(25)/gen(25) shows speedup gained from not using passage reranker (denoted as rr). Row gen(100)/rr+gen(25) compares the speed of using rr and gen(25) instead of  gen(100) (with no passage reranking). Results marked with $^*$ are not affected by passage reranking component, as they only measure speed of pipeline's individual component. For instance, table shows that doing answer reranking with generative reader with just 25 passages at its input runs 3.75x faster than doing answer reranking with generative reader that uses 100 passages.}
    \label{tab:inference_ratios}
\end{table*}

\end{document}